\documentclass{article}
\usepackage{ijcai25}

\usepackage{times}
\usepackage{soul}
\usepackage{url}
\usepackage[hidelinks]{hyperref}
\usepackage[utf8]{inputenc}
\usepackage[small]{caption}
\usepackage{graphicx}
\usepackage{amsmath}
\usepackage{amsthm}
\usepackage{booktabs}
\usepackage{algorithm}
\usepackage{algorithmic}
\usepackage[switch]{lineno}

\usepackage{amsmath}
\usepackage{amssymb}
\usepackage{amsthm}
\usepackage[algo2e]{algorithm2e}

\usepackage[utf8]{inputenc}
\usepackage{subfigure}
\usepackage{multirow}
\usepackage{pbox}
\usepackage{graphicx}
\usepackage{bm}
\usepackage{siunitx}
\usepackage{booktabs}
\usepackage{lipsum}
\usepackage{caption}
\usepackage{cancel}
\usepackage{mathabx}
\usepackage[toc,page,header]{appendix}
\usepackage{minitoc}

\usepackage[textsize=tiny]{todonotes}
\usepackage{marginnote}
 
\usepackage{listings}
\definecolor{codegreen}{rgb}{0,0.6,0}
\definecolor{codegray}{rgb}{0.5,0.5,0.5}
\definecolor{codepurple}{rgb}{0.58,0,0.82}
\lstdefinestyle{mystyle}{
    commentstyle=\color{codegreen},
    keywordstyle=\color{magenta},
    numberstyle=\tiny\color{codegray},
    stringstyle=\color{codepurple},
    basicstyle=\ttfamily\footnotesize,
    breakatwhitespace=false,         
    breaklines=true,                 
    captionpos=b,                    
    keepspaces=true,                 
    numbers=left,                    
    numbersep=5pt,
    showspaces=false,                
    showstringspaces=false,
    showtabs=false,                  
    tabsize=2
}
\definecolor{diffstart}{named}{gray}
\definecolor{diffincl}{named}{green}
\definecolor{diffrem}{named}{orange}
\lstdefinelanguage{diff}{
basicstyle=\ttfamily\small,
morecomment=[f][\color{diffstart}]{@@},
morecomment=[f][\color{diffincl}]{+\ },
morecomment=[f][\color{diffrem}]{-\ },
}
\usepackage{pifont}
\newtheorem{theorem}{Theorem}

\theoremstyle{plain}
\newtheorem*{customthm}{Theorem 1}
\newtheorem*{customthm2}{Theorem 2}

\theoremstyle{definition}

\theoremstyle{remark}

\title{Learning Advanced Self-Attention for Linear Transformers \\in the Singular Value Domain}

\author{
Hyowon Wi$^1$
\and
Jeongwhan Choi$^2$\And
Noseong Park$^1$
\affiliations
$^1$Korea Advanced Institute of Science and Technology (KAIST)\\ %IJCAI는 full form으로!
$^2$Yonsei University\\
\emails
hyowon.wi@kaist.ac.kr,
jeongwhan.choi@yonsei.ac.kr,
noseong@kaist.ac.kr
}

\begin{document}

\maketitle

\begin{abstract}
Transformers have demonstrated remarkable performance across diverse domains. The key component of Transformers is self-attention, which learns the relationship between any two tokens in the input sequence. Recent studies have revealed that the self-attention can be understood as a normalized adjacency matrix of a graph. Notably, from the perspective of graph signal processing (GSP), the self-attention can be equivalently defined as a simple graph filter, applying GSP using the value vector as the signal. However, the self-attention is a graph filter defined with only the first order of the polynomial matrix, and acts as a low-pass filter preventing the effective leverage of various frequency information. Consequently, existing self-attention mechanisms are designed in a rather simplified manner. Therefore, we propose a novel method, called \underline{\textbf{A}}ttentive \underline{\textbf{G}}raph \underline{\textbf{F}}ilter (AGF), interpreting the self-attention as learning the graph filter in the singular value domain from the perspective of graph signal processing for directed graphs with the linear complexity w.r.t. the input length $n$, i.e., $\mathcal{O}(nd^2)$. In our experiments, we demonstrate that AGF achieves state-of-the-art performance on various tasks, including Long Range Arena benchmark and time series classification.
\end{abstract}

\section{Introduction}
Transformers~\cite{vaswani2017attention} have achieved great success in many fields, including computer vision~\cite{touvron2021deit,liu2021swin}, time series analysis~\cite{li2019enhancing,wu2021autoformer,zhou2021informer}, natural language processing~\cite{nangia2018listops,maas2011text,radford2019gpt2,devlin2019bert}, and many other works~\cite{xiong2021nystromformer,wang2020linformer,yu2024learning,kim2024polynomial}. Many researchers agree that the self-attention mechanism plays a major role in the powerful performance of Transformers. The self-attention mechanism employs a dot-product operation to calculate the similarity between any two tokens of the input sequence, allowing all other tokens to be attended when updating one token.

\begin{figure}[t]
\centering
\subfigure[Vanilla self-attention in Transformers]{\includegraphics[width=0.9\columnwidth]{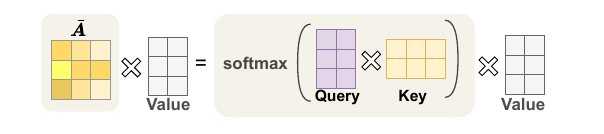}}
\subfigure[Graph signal processing on undirected graph]{\includegraphics[width=0.9\columnwidth]{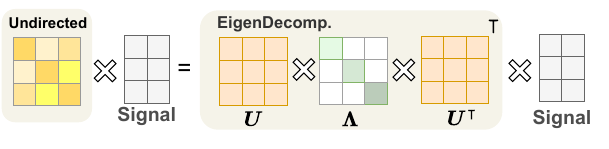}}
\subfigure[Graph signal processing on directed graph]{\includegraphics[width=0.9\columnwidth]{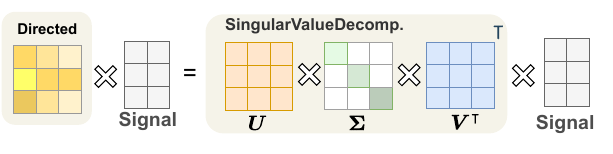}}
\caption{Illustration of the vanilla self-attention in Transformers and graph signal processing: (a) shows the softmax operation after the dot-product of query and key vectors, followed by the multiplication with the value vector; (b) and (c) show graph signal processing in undirected and directed graphs, respectively. For undirected graphs, the signal is filtered in the eigenvalue domain, while for directed graphs, the signal is filtered in the \emph{singular value domain}.}
\label{fig:sa_vs_gsp} 
\end{figure}

\paragraph{Linear Transformers approximate the self-attention.} However, the self-attention requires quadratic complexity over the input length to calculate the cosine similarity between any two tokens. This makes the self-attention difficult to apply to inputs with long lengths. In order to process long sequences, therefore, reducing the complexity of self-attention has become a top-priority goal, leading to the proposal of approximating the self-attention with linear complexity~\cite{child2019generating,ravula2020etc,zaheer2020bigbird,beltagy2020longformer,wang2020linformer,katharopoulos2020lineartransformer,choromanski2020performer,shen2021efficient,xiong2021nystromformer,qin2022cosformer,chen2023primal}. However, existing self-attention with linear complexity aims to create a matrix that is close to the original self-attention map $\bar{\mathbf{A}}$.

\paragraph{The self-attention is a low-pass filter (see Thm.~\ref{thm:sa_is_lpf}).} The self-attention map is a matrix that represents the relationship between every pair of tokens as scores normalized to probability by softmax. Considering each token as a node and the attention score as an edge weight, self-attention is considered as a normalized adjacency matrix from the perspective of the graph~\cite{shi2022revisiting,wang2022anti,choi2024GFSA}. Therefore, the self-attention plays an important role in the message passing scheme by determining which nodes to exchange information with. Given that the self-attention operates as a normalized adjacency matrix, it is intuitively aligned with graph signal processing (GSP)~\cite{ortega2018gsp,marques2020dgsp,chien2021GPRGNN,defferrard2016chebnet}, which employs graph structures to analyze and process signals. Fig.~\ref{fig:sa_vs_gsp} illustrates the self-attention in Transformers and signal filtering in GSP. As depicted in Fig.~\ref{fig:sa_vs_gsp} (a), the original self-attention takes the softmax of the output from the dot product of the query vector and the key vector, and then multiplies it by the value vector. Fig.~\ref{fig:sa_vs_gsp} (b) illustrates a general GSP method that applies the graph Fourier transform to the signal, conducts filtering in the spectral domain, and subsequently restores it to the original signal domain. In GSP, a signals are filtered through the graph filters, which are generally approximated by a matrix polynomial expansion. The self-attention mechanism in Transformers can be viewed as a graph filter $\mathbf{H}$ defined with only the first order of the polynomial matrix, i.e., $\mathbf{H}=\mathbf{\bar{A}}$. Furthermore, since the self-attention is normalized by softmax and functions as a low-pass filter (see Theorem.~\ref{thm:sa_is_lpf}), the high-frequency information in the value vector is attenuated, preventing the effective leverage of various frequency information. Consequently, existing self-attention mechanisms are designed in a rather simplified manner.

\paragraph{Our proposed linear Transformer learns an advanced self-attention (see Thm.~\ref{thm:coeff}).} Although the approximation of linear Transformers is successful, what they are doing is simply a low-pass filtering. Therefore, to increase the expressive power of linear Transformers, we propose a more generalized GSP-based self-attention, called \underline{\textbf{A}}ttentive \underline{\textbf{G}}raph \underline{\textbf{F}}ilter (AGF). We interpret the value vector of Transformers as a signal and redesign the self-attention as a graph filter. However, the existing self-attention mechanism possesses two problems: i) since the self-attention is based on directed graphs, the graph Fourier transform through eigendecomposition is not always guaranteed, and ii) the attention map change for every batch, making it too costly to perform the graph Fourier transform every batch. In order to address the first problem, therefore, we design a self-attention layer based on the GSP process in the \emph{singular value domain} (see Fig.~\ref{fig:sa_vs_gsp} (c)). The singular value decomposition (SVD) has been used recently for the GSP in directed graphs and can substitute the eigendecomposition~\cite{maskey2023fractional}. In order to address the second problem, we directly learn the singular values and vectors instead of explicitly decomposing the self-attention map or any matrix. Since our proposed self-attention layer directly learns in the \emph{singular value domain} by generating singular vectors and values using a neural network, our proposed method has a linear complexity of $O(nd^2)$. Therefore, our method efficiently handles inputs with long sequences.

Our contributions can be summarized as follows:
\begin{enumerate}
    \item We propose an advanced self-attention mechanism based on the perspective of signal processing on directed graphs, called \underline{\textbf{A}}ttentive \underline{\textbf{G}}raph \underline{\textbf{F}}ilter (AGF), motivated that the self-attention is a simple graph filter and acts as a low pass filter in the singular value domain.
    \item AGF learns a sophisticated graph filter directly in the singular value domain with linear complexity w.r.t. input length, which incorporates both low and high-frequency information from hidden representations.
    \item The experimental results for time series, long sequence modeling and image domains demonstrate that AGF outperforms existing linear Transformers.
    \item As a side contribution, we conduct additional experiments to show that AGF effectively mitigates the over-smoothing problem in deep Transformer models, where the hidden representations of tokens to become indistinguishable from one another.
\end{enumerate}

\section{Background}

\subsection{Self-attention in Transformer}

A key operation of Transformers is the self-attention which allows them to learn the relationship among tokens. The self-attention mechanism, denoted as SA: $\mathbb{R}^{n \times d} \rightarrow \mathbb{R}^{n \times d}$, can be expressed as follows:
\begin{align}\label{eq:SA}
   \textrm{SA}(\mathbf{X}) = \textrm{softmax} \Big( \frac{\mathbf{X}\mathbf{W}_{\textrm{qry}}(\mathbf{X}\mathbf{W}_{\textrm{key}})^\intercal}{\sqrt{d}} \Big) \mathbf{X}\mathbf{W}_{\textrm{val}} = \bar{\mathbf{A}}\mathbf{X}\mathbf{W}_{\textrm{val}},
\end{align} where $\mathbf{X}\in\mathbb{R}^{n \times d}$ is the input feature and $\widebar{\mathbf{A}}\in\mathbb{R}^{n \times n}$ is the self-attention matrix. $\mathbf{W}_{\textrm{qry}}\in\mathbb{R}^{d \times d}$, $\mathbf{W}_{\textrm{key}}\in\mathbb{R}^{d \times d}$, and $\mathbf{W}_{\textrm{val}}\in\mathbb{R}^{d \times d}$ are the query, key, and value trainable parameters, respectively, and $d$ is the dimension of token. The self-attention effectively learns the interactions of all token pairs and have shown reliable performance in various tasks.

However, in the case of the existing self-attention, a dot-product is used to calculate the attention score for all token pairs. To construct the self-attention matrix $\widebar{\mathbf{A}} \in \mathbb{R}^{n \times n}$, the matrix multiplication with query and key parameters mainly causes a quadratic complexity of $\mathcal{O}(n^2d)$. Therefore, it is not suitable if the length of the input sequence is large. This is one of the major computational bottlenecks in Transformers. For instance, BERT~\cite{devlin2019bert}, one of the state-of-the-art Large Language Model (LLM), needs 16 TPUs for pre-training and 64 TPUs with large models.

\subsection{Linear Transformer}\label{sec:related_work}
To overcome the quadratic computational complexity of the self-attention, efficient Transformer variants have been proposed in recent years. Recent research focuses on reducing the complexity of the self-attention in two streams. The first research line is to replace the softmax operation in the self-attention with other operations. 
For simplicity, we denote $\textrm{softmax}(\mathbf{X}\mathbf{W}_{qry}(\mathbf{X}\mathbf{W}_{key})^\intercal)$ as $\textrm{softmax}(\mathbf{Q}\mathbf{K}^\intercal)$.
\cite{wang2020linformer} introduce projection layers that map value and key vectors to low dimensions. \cite{katharopoulos2020lineartransformer} interprets softmax as kernel function and replace the similarity function with $\textrm{elu}(\mathbf{x})+1$. \cite{choromanski2020performer} approximates the self-attention matrix with random features. \cite{shen2021efficient} decomposes the $\textrm{softmax}(\mathbf{QK}^\intercal)$ into $\textrm{softmax}(\mathbf{Q})\textrm{softmax}(\mathbf{K}^\intercal)$, which allows to calculate $\textrm{softmax}(\mathbf{K}^\intercal)\mathbf{V}$ first, reducing the complexity from $\mathcal{O}(n^2d)$ to $\mathcal{O}(nd^2)$. \cite{qin2022cosformer} replaces softmax with a linear operator and adopts a cosine-based distance re-weighting mechanism. \cite{xiong2021nystromformer} adopts Nystr{\"o}m method by down sampling the queries and keys in the attention matrix. \cite{chen2023primal} employs an asymmetric kernel SVD motivated by low-rank property of the self-attention. However, these approaches sacrifice the performance to directly reduce quadratic complexity to linear complexity.

The second research line is to introduce sparsity in the self-attention. \cite{zaheer2020bigbird} introduces a sparse attention mechanism optimized for long document processing, combining local, random, and global attention to reduce computational complexity while maintaining performance. \cite{kitaev2020reformer} use locality-sensitive hashing and reversible feed forward network for sparse approximation, while requiring to re-implement the gradient back propagation. \cite{beltagy2020longformer} employ the self-attention on both a local context and a global context to introduce sparsity. \cite{zeng2021yosoe} take a Bernoulli sampling attention mechanism based on locality sensitive hashing. However, since they do not directly reduce the complexity to linear, they also suffer a large performance degradation, while having only limited additional speed-up.

\subsection{Graph Convolutional Filter}

The graph signal processing (GSP) can be considered as a generalized concept of the discrete signal processing (DSP). In the definition of DSP, the discrete signal with length $n$ is represented by the vector $\mathbf{x} \in \mathbb{R}^n$. Then for the signal filter $\mathbf{g} \in \mathbb{R}^n$ that transforms $\mathbf{x}$, the convolution operation $\mathbf{x} * \mathbf{g}$ is defined as follows:
\begin{align}
    y_i = \sum_{j=1}^n \mathbf{x}_j\mathbf{g}_{i-j},
\end{align} where the index $i$ indicates the $i$-th element of each vector. 
GSP extends DSP to signal samples indexed by nodes of arbitrary graphs. Then we define the shift-invariant graph convolution filters $\mathbf{H}$ with a polynomial of graph shift operator $\mathbf{S}$ as follows: 
\begin{align}\label{eq:GSP}
    \mathbf{y} = \mathbf{H}\mathbf{x} = \sum_{k=0}^{K} w_k\mathbf{S}^k\mathbf{x},
\end{align} 

where $K$ is the maximum order of polynomial and $w_k\in [-\infty, \infty]$ is a coefficient. The graph filter is parameterized as the truncated expansion with the order of $K$. The most commonly used graph shift operators in GSP are adjacency and Laplacian matrices. Note that Eq.~\eqref{eq:GSP} applies to any directed or undirected adjacency matrix~\cite{ortega2018gsp,marques2020dgsp}. However, Eq.~\eqref{eq:GSP} requires non-trivial matrix power computation. Therefore, we rely on SVD to use the more efficient way in Eq.~\eqref{eq:graph_filter_monomial}.

\section{Proposed Method}\label{sec:method}

\begin{figure*}[t]
    \centering
    \includegraphics[width=0.82\textwidth]{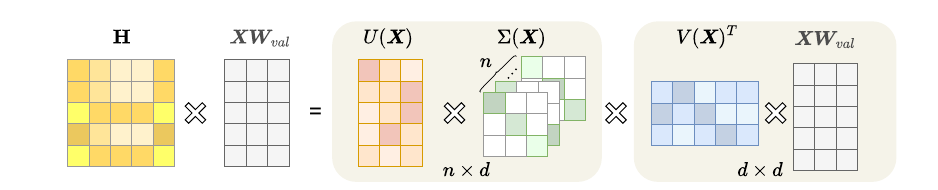}
    \caption{The proposed AGF performs the directed GSP in the singular value domain by learning $U(\mathbf{X})$, $\Sigma(\mathbf{\mathbf{X}})$, and $V(\mathbf{X})$ (cf. Eqs.~\eqref{eq:singular} to~\eqref{eq:singular2}). The $n$ different sets of singular values in $\Sigma(\mathbf{X})$ are used for token-specific processing. In other words, $n$ different graph filters are used for $n$ different tokens in order to increase the representation learning capability of AGF.}
    \label{fig:overall_architecture}
\end{figure*}

\subsection{Self-attention as a graph filter}
    
The self-attention learns the relationship among all token pairs. From a graph perspective, each token can be interpreted as a graph node and each self-attention score as an edge weight. Therefore, self-attention produces a special case of the normalized adjacency matrix~\cite{shi2022revisiting,wang2022anti} and can be analyzed from the perspective of graph signal processing (GSP). In GSP, the low-/high-frequency components of a signal $\mathbf{x}$ are defined using the Discrete Fourier Transform (DFT) $\mathcal{F}$ and its inverse $\mathcal{F}^{-1}$. Let $\bar{\mathbf{x}} = \mathcal{F}\mathbf{x}$ denote the spectrum of $\mathbf{x}$. Then, $\bar{\mathbf{x}}_{\text{lfc}} \in \mathbb{C}^c$ contains the $c$ lowest-frequency components of $\bar{\mathbf{x}}$, and $\bar{\mathbf{x}}_{\text{hfc}} \in \mathbb{C}^{n-c}$ contains the remaining higher-frequency components. The low-frequency components (LFC) of $\mathbf{x}$ are given as $\text{LFC}[\mathbf{x}] = \mathcal{F}^{-1}(\bar{\mathbf{x}}_{\text{lfc}}) \in \mathbb{R}^n$, and the high-frequency components (HFC) are defined as $\text{HFC}[\mathbf{x}] = \mathcal{F}^{-1}(\bar{\mathbf{x}}_{\text{hfc}}) \in \mathbb{R}^n$. Here, the DFT $\mathcal{F}$ projects $\mathbf{x}$ into the frequency domain, and $\mathcal{F}^{-1}$ reconstructs $\mathbf{x}$ from its spectrum. The Fourier basis $\bm{f}_j = [e^{2\pi i(j-1)\cdot 0}, e^{2\pi i(j-1)\cdot 1}, \dots, e^{2\pi i(j-1)(n-1)}]^{\intercal} / \sqrt{n}$ is used in computing $\mathcal{F}$, where $j$ denotes the $j$-th row. In GSP, the adjacency matrix functions as a low-pass filter, using the edge weights to aggregate information from nodes attenuates the high-frequency information of the nodes. In other words, the self-attention also acts as a low-pass filter within Transformers, and it is theoretically demonstrated below.

\begin{theorem}[Self-attention is a low-pass filter]\label{thm:sa_is_lpf}

Let $\mathbf{M}=\textrm{softmax}(\mathbf{Z})$ for any matrix $\mathbf{Z}\in\mathbb{R}^{n\times n}$. Then $\mathbf{M}$ inherently acts as a low pass filter. For all $\mathbf{x}\in \mathbb{R}^N$, in other words, $\text{lim}_{t\rightarrow \infty}\Vert \text{HFC}[\mathbf{M}^t(\mathbf{x})]\Vert_2/\Vert \text{LFC}[\mathbf{M}^t(\mathbf{x})]\Vert_2=0$
\end{theorem}
The proof of Theorem~\ref{thm:sa_is_lpf} is in Appendix~\ref{app:proof_sa_is_lpf}. As the self-attention is normalized by softmax, the self-attention functions as a low-pass filter. Hence, Transformers are unable to sufficiently leverage a various scale of frequency information, which reduces the expressive power of Transformers.

Inspired by this observation, we redesign a graph filter-based self-attention from the perspective of GSP. As mentioned earlier, since the adjacency matrix can serve as a graph-shift operator, it is reasonable to interpret the self-attention as a graph-shift operator, $\mathbf{S}=\mathbf{\bar{A}}$. Moreover, the self-attention block of the Transformer is equivalent to defining a simple graph filter  $\mathbf{H}=\mathbf{\bar{A}}$ and applying GSP to the value vector, treated as a signal. Therefore, in Eq.~\eqref{eq:GSP}, we can design a more complex graph filter through the polynomial expansion of the self-attention.

Note that when we interpret the self-attention $\mathbf{\bar{A}}$ as a graph, it has the following characteristics: i) The self-attention is an asymmetric directed graph, and ii) all nodes in the graph are connected to each other since the self-attention calculates the relationships among the tokens. Then we can derive that the self-attention is a special case of the symmetrically normalized adjacency (SNA) as $\mathbf{\bar{A}}=\mathbf{D}^{-1}\mathbf{A}$ where $\mathbf{A}$ is an adjacency matrix and $\mathbf{D}$ is a degree matrix of nodes. In particular, SNA is one of the most popular forms for directed GSP~\cite{maskey2023fractional}.

\subsection{Polynomial graph filter}
When approximating the graph filter with a matrix polynomial, $k-1$ matrix multiplications are required to calculate up to the $k$-th order (cf. Eq.~\eqref{eq:GSP}), which requires a large computational cost. Therefore, to reduce the computational complexity, we avoid matrix multiplications by directly learning the graph filter in the spectral domain. In the case of an undirected graph, a filter can be learned in the spectral domain by performing the graph Fourier transform through eigendecomposition. In general, however, the eigendecomposition is not guaranteed for directed graphs. The GSP through SVD, therefore, is often used~\cite{maskey2023fractional} if i) a directed graph $\widebar{\mathbf{A}}$ is SNA and ii) its singular values are non-negative and within the unit circle, i.e., $\Vert \widebar{\mathbf{A}} \Vert \leq 1$. For $\widebar{\mathbf{A}}$ and its SVD $\widebar{\mathbf{A}}=\mathbf{U}\mathbf{\Sigma} \mathbf{V}^\intercal$, an $\alpha$-power of the symmetrically normalized adjacency is defined as:
\begin{align}
    \widebar{\mathbf{A}}^\alpha:=\mathbf{U}\mathbf{\Sigma}^\alpha \mathbf{V}^\intercal,
\end{align}where $\alpha \in \mathbb{R}$~\cite{maskey2023fractional}. Therefore, we can define the graph filter $\mathbf{H}$ as follows:
\begin{align}\label{eq:graph_filter_monomial}
\mathbf{y}=\mathbf{H}\mathbf{x} = g_\theta(\mathbf{\bar{A}})\mathbf{x} = \mathbf{U}g_\theta(\mathbf{\Sigma})\mathbf{V}^\intercal\mathbf{x} = \mathbf{U}(\sum_{k=0}^{K}\theta_k\mathbf{\Sigma}^k)\mathbf{V}^{\intercal}\mathbf{x},
\end{align}
where $\theta\in\mathbb{R}^n$ is a vector for singular value coefficients. Therefore, a spectral filter can be defined as a truncated expansion with $K$-th order polynomials. In other words, unlike directly performing the matrix polynomial as in Eq.~\eqref{eq:GSP}, the computational cost is significantly reduced by $K$ times element-wise multiplying of the singular values, which are represented as a diagonal matrix.

However, the polynomial expansion in Eq.~\eqref{eq:graph_filter_monomial} is parameterized with monomial basis, which is unstable in terms of its convergence since the set of bases is non-orthogonal. Therefore, for stable convergence, a filter can be designed using an orthogonal basis. Note that we have the flexibility to apply any basis when using the polynomial expansion for learning graph filters. In this work, we adopt the Jacobi expansion~\cite{askey1985jacobi}, one of the most commonly used polynomial bases. Furthermore, Jacobi basis is a generalized form of classical polynomial bases such as Chebyshev~\cite{defferrard2016chebnet} and Legendre~\cite{mccarthy1993legendre}, offering strong expressiveness in the graph filter design. Detailed formulas are provided in Appendix~\ref{app:jacobi}. Therefore, we can define the graph polynomial filter as follows:
\begin{align}\label{eq:graph_filter_general}
g_\theta(\mathbf{\Sigma})=\sum_{k=0}^{K}\theta_kT_k(\mathbf{\Sigma}),
\end{align} where $T_k(\cdot)$ is a specific polynomial basis of order $k$.

\subsection{Attentive Graph Filter}
In order to use Eq.~\eqref{eq:graph_filter_general}, however, we need to decompose the adjacency matrix $\widebar{\mathbf{A}}$, which incurs non-trivial computation. Therefore, we propose to directly learn a graph filter in the singular value domain (instead of learning an adjacency matrix, i.e., a self-attention matrix, and decomposing it). Therefore, as shown in Fig.~\ref{fig:overall_architecture}, we propose our attentive graph filter (AGF) as follows:
\begin{align}
    \mathrm{AGF}(\mathbf{X}) &= \mathbf{HXW}_{\textrm{val}} =U(\mathbf{X})\Sigma(\mathbf{\mathbf{X}})V(\mathbf{X})^\intercal\mathbf{X}\mathbf{W}_{\textrm{val}},
\end{align}
\begin{align}\label{eq:singular}
    U(\mathbf{X}) &= \rho(\mathbf{X}\mathbf{W}_U) & \in \mathbb{R}^{n\times d}, \\
    \Sigma(\mathbf{X}) &=\sum_{k=0}^{K} \theta_k T_k(\mathrm{diag}(\sigma(\mathbf{X}\mathbf{W}_\Sigma))) & \in \mathbb{R}^{n\times d \times d}, \label{eq:sigma}\\    V(\mathbf{X})^\intercal &= \rho((\mathbf{X}\mathbf{W}_V)^\intercal) & \in \mathbb{R}^{d\times n}, \label{eq:singular2}
\end{align} where $\mathbf{W}_U, \mathbf{W}_\Sigma, \mathbf{W}_V\in\mathbb{R}^{d \times d}$ are learnable matrices, $\rho$ is a softmax, and $\sigma$ is a sigmoid. Our proposed model does not apply SVD directly on the computed self-attention or other matrices. Instead, the learnable singular values $\sigma(\mathbf{X}\mathbf{W}_\Sigma)$ and orthogonally regularized singular vectors $U(\mathbf{X})$ and $V(\mathbf{X})$ are generated by neural network. The singular values are then filtered by the graph filter, denoted as $\Sigma(\mathbf{X})$. To ensure that the elements of the singular value matrix are non-negative and within the unit circle, the sigmoid function is applied to the matrix. Moreover, we observe that the softmax of singular vectors enhances the stability of learning.

We construct our graph filter using the generated singular values, leveraging the Jacobi expansion as an orthogonal polynomial basis. If the trainable coefficients $\theta_k$ is allowed to take negative values and learned adaptively, the graph filter can pass low/high-frequency components of the value vector. Therefore, AGF functions as a graph filter that leverages various frequency information from the value vector. Furthermore, unlike the adjacency matrix that remains unchanged in GCNs, the self-attention matrix changes with each batch. To enhance the capacity for addressing these dynamics, AGF incorporates a token-specific graph filter, characterized by $n$ different sets of singular values. This allows to leverage the token-specific frequency information in the singular value domain, increasing the capability to handle complex dynamics in hidden representation.

\subsection{Objective Function}

In the definition of SVD, $U(\mathbf{X})$ is column orthogonal, $V(\mathbf{X})$ is row orthogonal, and $\Sigma(\mathbf{X})$ is a rectangular diagonal matrix with non-negative real numbers. When we train the proposed model, strictly constraining $U(\mathbf{X})$ and $V(\mathbf{X})$ to be orthogonal requires a high computational cost. Instead, we add a regularization on them since these matrices generated by neural network can be trained to be orthogonal as follows:
\begin{align}\begin{split}
\mathcal{L}_{ortho} = &\frac{1}{n^2} \big(\Vert (U(\mathbf{X})^\intercal U(\mathbf{X})-\mathbf{I} \Vert + \Vert (V(\mathbf{X})V(\mathbf{X})^\intercal-\mathbf{I} \Vert \big),
\end{split}\end{align} where $\mathbf{I}\in\mathbb{R}^{d \times d}$ is an identity matrix. Therefore, our joint learning objective $\mathcal{L}$ is as follows:
\begin{align}
\mathcal{L} =  \mathcal{L}_{transformer} + \gamma \mathcal{L}_{ortho},
\end{align} where $ \mathcal{L}_{transformer}$ is an original objective function for Transformers. The hyperparameters $\gamma$ controls the trade-off between the loss and the regularization.

\subsection{Time and Space Complexities of AGF}
Since our AGF is based on the concept of SVD, it is not restricted by softmax for calculating attention scores. Therefore, $U(\mathbf{X})$, $\Sigma(\mathbf{X})$, and $V(\mathbf{X})$ generated by neural network can be freely multiplied according to the combination law of matrix multiplication. First, since $\Sigma(\mathbf{X})$ is a diagonal matrix, by performing element-wise multiplication with $U(\mathbf{X})$ and the diagonal elements of $\Sigma(\mathbf{X})$, $(n \times d)$ matrix is calculated with a time complexity of $\mathcal{O}(nd)$. Next, by multiplying $V(\mathbf{X})$ and the value vector, $(d \times d)$ matrix is calculated with a time complexity of $\mathcal{O}(nd^2)$. Finally, by multiplying the outputs of steps 1 and 2, the final output is $(n \times d)$ matrix with a time complexity of $\mathcal{O}(nd^2)$. Therefore, the time complexity is $\mathcal{O}(nd^2)$ and the space complexity is $\mathcal{O}(nd+ d^2)$.

\subsection{Properties of AGF}

\paragraph{How to use high-frequency information.}
In GSP, the characteristics of the graph filter are determined by the learned coefficients $\theta_k$ of the signal. These coefficients allow the graph filter to function as a low-pass, high-pass, or combined-pass filter, depending on the specific needs of each task~\cite{defferrard2016chebnet,marques2020dgsp,chien2021GPRGNN}, demonstrated by following theorem:
\begin{theorem}[Adapted from \cite{chien2021GPRGNN}]\label{thm:coeff}
Assume that the graph $G$ is connected. If $\theta_k \geq 0$ for $\forall k \in \{0, 1, ..., K\}$, $\sum_{k=0}^K \theta_k =1$ and $\exists k'> 0$ such that $\theta_{k'}> 0$, then $g_\theta(\cdot)$ is a low-pass graph filter. Also, if $\theta_k=(-\alpha)^k$ , $\alpha \in (0, 1)$ and $K$ is large enough, then $g_\theta(\cdot)$ is a high-pass graph filter.
\end{theorem}

The proof is in Appendix~\ref{app:proof_coeff}. Theorem~\ref{thm:coeff} indicates that if the coefficient $\theta_k$ of a graph filter can have negative values, and learned adaptively, the graph filter will pass low and high frequency signals appropriately. This flexibility is crucial for effectively processing signals with varying frequency components. Similarly, AGF operates as a filter that modulates frequency information in the singular value domain through the generated singular values and singular vectors. This approach enables AGF to dynamically adjust the frequency components of the signal, providing a more tailored and efficient filtering process. Therefore, unlike conventional Transformers, AGF can appropriately incorporate both low and high frequencies for each task, thereby enhancing the expressive power and adaptability of Transformers.

\paragraph{Comparison with existing linear self-attention methods.}
We explain that while our AGF addresses the computational inefficiencies inherent in the vanilla self-attention like existing linear self-attention studies, we take a different approach from them. Instead of using explicit SVDs, our AGF reinterprets self-attention through a GSP lens, using the learnable SVD to learn graph filters directly from the spectral domain of directed graphs. Linformer~\cite{wang2020linformer}, the most prominent representative of linear self-attention, approximates the vanilla self-attention through dimensionality reduction, and Nyströmformer~\cite{xiong2021nystromformer}, which reduces to linear complexity with a kernel decomposition method, also efficiently approximates the full self-attention matrix with the Nyström method. Singularformer~\cite{wu2023singularformer}, a closely related approach, uses a parameterized SVD and linearize the calculation of self-attention. However, like existing linear Transformers, it approximates the original self-attention, which is inherently a low-pass filter. Thus, to the best of our knowledge, existing linear self-attention methods focus on approximating the self-attention and reducing it to linear complexity, whereas our AGF approximates a graph filter rather than the self-attention. This allows AGF to use the token-specific graph filter to improve model representation within the singular value domain. 

\section{Experiments}\label{sec:experiments}
\begin{table*}[t]
\centering
\resizebox{0.87\textwidth}{!}{%
\begin{tabular}{l  cccccccccc  c}\toprule
                  & \multicolumn{1}{l}{EC} & \multicolumn{1}{l}{FD} & \multicolumn{1}{l}{HW} & \multicolumn{1}{l}{HB} & \multicolumn{1}{l}{JV} & \multicolumn{1}{l}{PEMS-SF} & \multicolumn{1}{l}{SRSCP1} & \multicolumn{1}{l}{SRSCP2} & \multicolumn{1}{l}{SAD} & \multicolumn{1}{l}{UWGL} & \multicolumn{1}{c}{Avg} \\ \midrule
Transformer       & 32.7 & 67.3 & 32.0 & 76.1 & 98.7 & 82.1 & 92.2 & 53.9 & 98.4 & 85.6 & 71.9 \\
LinearTransformer & 31.9 & 67.0 & \underline{34.7} & 76.6 & \underline{99.2} & 82.1 & \underline{92.5} & 56.7 & 98.0 & 85.0 & 72.4 \\
Reformer          & 31.9 & \underline{68.6} & 27.4 & 77.1 & 97.8 & 82.7 & 90.4 & 56.7 & 97.0 & 85.6 & 71.5 \\
Longformer        & 32.3 & 62.6 & \textbf{39.6} & \underline{78.0} & 98.9 & 83.8 & 90.1 & 55.6 & 94.4 & 87.5 & 72.0 \\
Performer         & 31.2 & 67.0 & 32.1 & 75.6 & 98.1 & 80.9 & 91.5 & 56.7 & 98.4 & 85.3 & 71.9 \\
YOSO-E            & 31.2 & 67.3 & 30.9 & 76.5 & 98.6 & 85.2 & 91.1 & 53.9 & 98.9 & \underline{88.4} & 72.2 \\
Cosformer         & 32.3 & 64.8 & 28.9 & 77.1 & 98.3 & 83.2 & 91.1 & 55.0 & 98.4 & 85.6 & 71.5 \\
SOFT              & 33.5 & 67.1 & \underline{34.7} & 75.6 & \underline{99.2} & 80.9 & 91.8 & 55.6 & 98.8 & 85.0 & 72.2 \\
Flowformer        & \underline{33.8} & 67.6 & 33.8 & 77.6 & 98.9 & 83.8 & 92.5 & 56.1 & 98.8 & 86.6 & 73.0 \\ 
Primalformer      & 33.1 & 67.1 & 29.6 & 76.1 & 98.3 & \underline{89.6} & \underline{92.5} & \underline{57.2} & \textbf{100.0}         & 86.3 & \underline{73.0} \\ \midrule
AGF               & \textbf{36.1} & \textbf{69.9} & 33.5 & \textbf{79.0}& \textbf{99.5} & \textbf{91.3} & \textbf{93.5} & \textbf{58.9} & \textbf{100.0} & \textbf{89.4} & \textbf{75.1} \\ \bottomrule
\end{tabular}
}
\caption{Performance comparison on UEA time series classification. Abbreviations are as follows: EthanolConcentration (EC), FaceDetection (FD), HandWriting (HW), HearBeat (HB), JapaneseVowels (JV), PEMS-SF, SelfRegulation SCP1 (SRSCP1), SelfRegulation SCP2 (SRSCP2), SpokenArabicDigits (SAD), and UWaveGesture Library (UWGL).}
\label{tab:ts_main}
\end{table*}

\begin{table}[t]
\small
\centering
\setlength{\tabcolsep}{2.5pt}
\resizebox{0.95\columnwidth}{!}{%
\begin{tabular}{l  ccccc  c}\toprule
                    & \multicolumn{1}{l}{ListOps} & \multicolumn{1}{l}{Text} & \multicolumn{1}{l}{Retrieval} & \multicolumn{1}{l}{Image} & \multicolumn{1}{l}{Pathfinder} & \multicolumn{1}{l}{Avg} \\ \midrule
                    % & \multicolumn{1}{l}{ListOps(2K)} & \multicolumn{1}{l}{Text(4K)} & \multicolumn{1}{l}{Retrieval(4K)} & \multicolumn{1}{l}{Image(1K)} & \multicolumn{1}{l}{Pathfinder(1K)} & \multicolumn{1}{l}{Average} \\ \midrule
Transformer         & 37.1 & 65.0 & 79.4   & 38.2  & 74.2    & 58.8 \\
Reformer            & 19.1 & 64.9 & 78.6   & 43.3  & 69.4    & 55.1 \\
Performer           & 18.8 & 63.8 & 78.6   & 37.1  & 69.9    & 53.6 \\
Singularformer      & 18.7 &  61.8 & 76.7 & 35.3 &  55.8 & 49.7 \\
Linformer           & 37.3 & 55.9 & 79.4   & 37.8  & 67.6    & 55.6 \\
Nystr{\"o}mformer   & 37.2 & 65.5 & 79.6   & 41.6  & 70.9    & 59.0 \\
Longformer          & 37.2 & 64.6 & 81.0   & 39.1  & 73.0    & 59.0 \\
YOSO-E              & 37.3 & 64.7 & 81.2   & 39.8  & 72.9    & 59.2 \\
Primalformer &  37.3  & 61.2 & 77.8   & 43.0 & 68.3       & 57.5               \\  \midrule
AGF             & 38.0 & 64.7 & 81.4      & 42.4         &  74.0 &    \textbf{60.1} \\ \bottomrule

\end{tabular}
}
\caption{Performance comparison on LRA benchmark}
\label{table:lra_main}
\end{table}

\subsection{Time Series Classification}
\paragraph{Experimental settings.}\label{sec:exp_uea_setting}
To evaluate the performance of AGF, we employ UEA Time Series Classification Archive~\cite{bagnall2018uea} which is the benchmark on temporal sequences. Strictly following~\cite{wu2022flowformer}, we report accuracy for 10 multivariate datasets preprocessed according to~\cite{zerveas2021transformerts}. We adopt 2-layer Transformer as backbone with 512 hidden dimension on 8 heads and 64 embedding dimension of self-attention. The experiments are conducted on 1 GPU of NVIDIA RTX 3090 24GB. The detailed descriptions are in Appendix~\ref{app:exp_setting_uea}.

\paragraph{Experimental results.}
Table~\ref{tab:ts_main} summarizes the test accuracy of AGF and the state-of-the-art linear Transformer models on the UEA time series classification task. We observe that AGF achieves an average accuracy of 75.1, outperforming the vanilla Transformer and other linear Transformers by large margins across various datasets. This performance gap underscores the effectiveness of our approach in leveraging advanced graph filter-based self-attention to enhance the expressive power of Transformers.

% Specifically, AGF shows superior performance in Ethanol Concentration (EC) dataset with 36.1, compared to 33.8 achieved by Flowformer. 

\subsection{Long Range Arena Benchmark}
\paragraph{Experimental settings.}\label{sec:exp_lra_setting}
We evaluate AGF on Long Range Arena (LRA)~\cite{tay2020lra} benchmark under long-sequence scenarios. Following~\cite{xiong2021nystromformer}, we train 2 layer Transformer with 128 hidden dimension, 2 heads, and 64 embedding dimension with mean pooling. The experiments are conducted on 1 GPU of NVIDIA RTX 3090 24GB. The details are in Appendix~\ref{app:exp_setting_lra}.

\paragraph{Experimental results.}
We report the top-1 test accuracy on LRA benchmark in Table~\ref{table:lra_main}. Our modeldemonstrated the highest average performance, achieving a score of 60.1 --- an improvement of 1.3 points over the vanilla Transformer. In contrast, SingularFormer, a close approach that parameterizes SVD, only functions as a low-pass filter and thus fails to achieve optimal performance. Compared with YOSO-E, a state-of-the-art linear-complexity Transformer, AGF improves the performance by a substantial margin.

\subsection{Sensitivity Analyses}
We conduct sensitivity studies on $K$ and $\gamma$. For other sensitivity studies on are reported in Appendix~\ref{app:full_sens}.
\paragraph{Sensitivity study on $K$.}
We test our model by varying $K$ on UEA time series classification, and the results are shown in Table~\ref{tab:uea_k_sens}. As $K$ increases, the performance improves. However, beyond a certain threshold, increasing $K$ results in saturation and diminished performance. Therefore, choosing an appropriate $K$ has a significant impact on performance.

\paragraph{Sensitivity study on $\gamma$.}
Table~\ref{tab:lra_gamma_sens} summarizes the impact of $\gamma$ on LRA benchmark. The optimal level of regularization applied to learnable singular vectors varies depending on the dataset, and we demonstrate that imposing a certain degree of regularization can enhance training stability.

\begin{table}[t]
    \small
    \centering
    \setlength{\tabcolsep}{3.5pt}
    \begin{tabular}{l ccccccc}
    \toprule
    $K$ & EC & FD & JV & PEMS-SF & SRSCP1 & UWGL \\ \midrule
    3  & 32.3 & 68.2 & 98.9 & 86.7 & 91.1 & 84.1 \\
    4  & 31.6 & 68.8 & \textbf{99.5} & \textbf{89.6}  & 92.2 & 84.1 \\
    6  & \textbf{36.1} & 68.3 & 98.9 & 83.8 & 91.1 & 84.4 \\
    9  & 30.8 & \textbf{69.9} & 99.2 & 83.8 & \textbf{93.5} & 86.2 \\
    10 & 32.3 & 67.5 & 98.9 & 87.3 & 91.1 & \textbf{89.4} \\\bottomrule     
\end{tabular}
    \caption{Effect of $K$ on UEA classification} 
    \label{tab:uea_k_sens}
\end{table}

\begin{table}[t]
    \small
    \centering
    \setlength{\tabcolsep}{3pt}
    \begin{tabular}{l ccccc}
    \toprule
        $\gamma$ & ListOps & Text & Retrieval & Image & Pathfinder\\ \midrule
        $\num{1e-1}$ & \textbf{38.0} & 64.3 & \textbf{81.4} & 40.8 & 73.1 \\
        $\num{1e-2}$ & 36.9 & 64.5 & 79.8 & \textbf{42.4} & 73.3 \\ 
        $\num{1e-3}$ & 37.2 & \textbf{64.7} & 79.5 & 42.0 & \textbf{74.0} \\
        $\num{1e-4}$ & 37.0 & 64.2 & 79.4 & 41.0 & \textbf{74.0} \\
        \bottomrule     
    \end{tabular}
    \caption{Effect of $\gamma$ on LRA benchmark} 
    \label{tab:lra_gamma_sens}
\end{table}

\begin{table*}[t]
% \small
\centering
\resizebox{0.82\linewidth}{!}{%
\small
\begin{tabular}{l  ccccc  c }\toprule
                    & ListOps(2K) & Text(4K) & Retrieval(4K) & Image(1K) & Pathfinder(1K) & Avgerage \\ \midrule
Transformer  & 194.5/5.50 & 694.8/21.24 & 1333.7/18.72 & 334.5/5.88 & 405.5/5.88 & 592.6/11.44 \\
Nyströmformer     & 68.3/0.89 & 52.3/0.48 & 187.5/1.93 & 227.9/1.93 & 232.6/3.29 & 153.7/1.70  \\
Performer         & 90.3/1.67 & 55.9/0.84 & 230.7/3.34 & 296.7/3.34 & 344.8/6.28 & 203.7/3.09  \\
Reformer          & 94.1/1.64 & 58.1/0.82 & 244.2/3.29 & 309.1/3.29 & 370.7/6.09 & 215.2/3.03  \\
PrimalFormer & 56.5/0.69 & 93.6/1.37 & 185.3/2.99 & 142.9/1.39 & 180.0/1.52 & 131.7/1.59  \\ \midrule 
AGF          &60.8/0.88 &	48.4/0.51 &	252.3/3.95 &	183.3/2.15	 &209.3/1.89 &	150.8/1.90 \\ \bottomrule
\end{tabular}
}
\caption{Running time (s/1K-steps) and the peak training memory usage (GB) on LRA benchmark}
\label{tab:runtime_lra}
\end{table*}

\begin{table}[t]
\centering
\small
\setlength{\tabcolsep}{2.8pt}
\begin{tabular}{l  cccccccc}\toprule
         & \multicolumn{1}{l}{EC} & \multicolumn{1}{l}{FD} & \multicolumn{1}{l}{HW} & \multicolumn{1}{l}{HB}  & \multicolumn{1}{l}{PEMS-SF}  & \multicolumn{1}{l}{UWGL} \\ \midrule

$\mathbf{H}_{UV^\intercal}$  & 29.7 & 66.6 & 28.2 & 76.6  & 87.3  & 83.8  \\
$\mathbf{H}_{SVD}$ & 33.1 & 67.1 & 27.1 & 75.1  &   88.4   & 85.9\\ \midrule
AGF            & \textbf{36.1} & \textbf{69.9} & \textbf{33.5} & \textbf{79.0} & \textbf{91.3}  & \textbf{89.4}\\ \bottomrule
\end{tabular}
\caption{Ablation study on the graph filter}
\label{tab:abl_signal}
\end{table}

\subsection{Empirical Runtime}
Table~\ref{tab:runtime_lra} summarize the results of the runtime and peak memory usage during the training phase. AGF consistently improves the efficiency of both time and space complexity compared to the vanilla Transformer. Specifically, for Text dataset, which have extremely long input sequences, the efficiency of AGF stands out even more. When compared with other linear complexity Transformers, our AGF shows comparable efficiency with longer sequences.

\subsection{Ablation Studies}\label{subsec:abla}

We conduct various ablation studies, and the additional results on $T(\cdot)$ and $\rho$ are in Appendix~\ref{app:ablation}.

\paragraph{Effect on the graph filter.} 
To analyze the impact of graph filter, we conduct an ablation study on the following variants: i) $\mathbf{H}_{UV^\intercal}$ refers to the graph filters with parameterized singular vectors and the singular values are fixed as one; ii) $\mathbf{H}_{SVD}$ initializes the singular values to one, allowing them to be learnable from $\mathbf{H}_{UV^\intercal}$; and iii) AGF refers to the proposed method. Table~\ref{tab:abl_signal} shows the result of the effect of the graph filter, and in general, these ablation models leads to suboptimal performance. However, AGF processes the generated signal through the graph filter, allowing the model to use various scales of frequency information. The graph filter enhances the capacity of the model, resulting in optimal performance and demonstrating the effectiveness of AGF.

\subsection{Additional Experiments on Deep Transformer}\label{subsec:vit}

\begin{table}[]
    % \small
    \centering
    \setlength{\tabcolsep}{5pt}
    % \captionsetup{type=table}
    \begin{tabular}{l  cc} \toprule
       Model  &  ImageNet-100 & ImageNet-1K \\ \midrule
       DeiT-small       & 80.6 &     79.8   \\ 
       \ + AGF    & \textbf{81.3} &  \textbf{80.3}   \\ \bottomrule
    \end{tabular}
    \caption{Comparison of performance for DeiT-small trained on ImageNet-100 and ImageNet-1K}
    \label{tab:vit_main}
\end{table}

\begin{figure}[t]
        \centering
        \subfigure[Filter response]{\includegraphics[width=0.47\columnwidth]{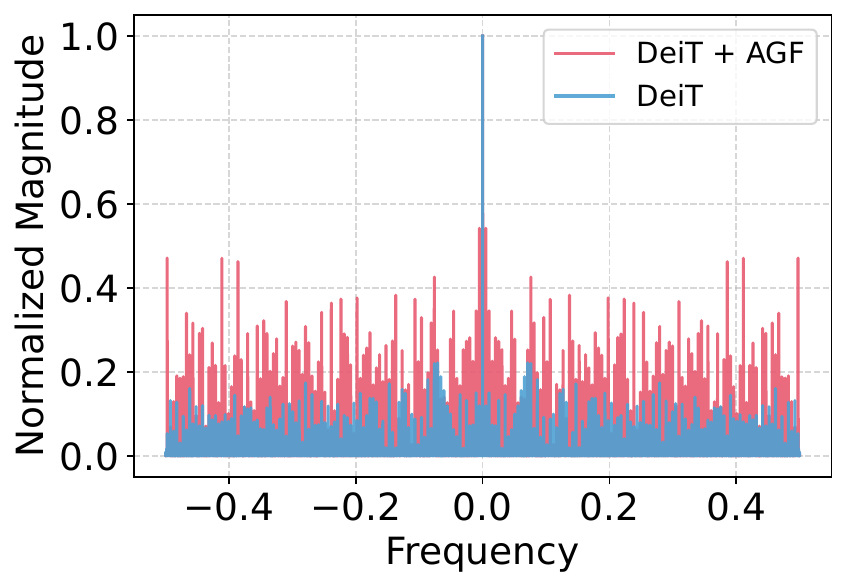}}
        \subfigure[Cosine similarity]{\includegraphics[width=0.47\columnwidth]{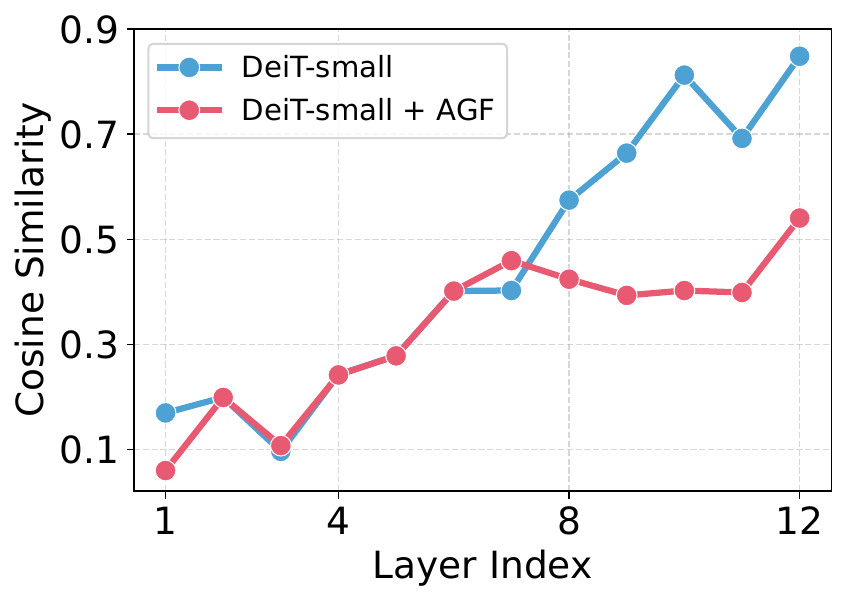}}
        \vspace{-1em}
        \caption{Filter response and cosine similarity on ImageNet-1k for DeiT-small and DeiT-small + AGF}
        \label{fig:oversmoothing_filter_response}
\end{figure}

\paragraph{Experimental settings.}\label{sec:exp_vit_setting}
We conduct additional experiments for image classification task with ImageNet-100~\cite{russakovsky2015imagenet100} and ImageNet-1K~\cite{deng2009imagenet1k} datasets and report top-1 accuracy. We adopt DeiT-small as the backbone, and trained from scratch with 300 epochs~\cite{touvron2021vittraining} with 2 GPU of NVIDIA RTX 3090 24GB. The detailed descriptions are in Appendix~\ref{app:exp_setting_vit}.

\paragraph{Experimental results.}
Table~\ref{tab:vit_main} shows the top-1 accuracy on ImageNet-100 and ImageNet-1k. Our AGF effectively learns the representation in deep layers model, which has 12 layers. Notably, plugging AGF improves the performance marginally, from 80.6 to 81.3 trained on ImageNet-100 datasets and from 79.8 to 80.3 on ImageNet-1K.

\paragraph{Analysis on mitigating over-smoothing problem.}
Deep Transformers, like GCNs, suffers from over-smoothing problem, where hidden representations become similar and indistinguishable to the last layer~\cite{kipf2017GCN,velickovic2018GAT,oono2020oversmoothing,rusch2023survey}. We previously demonstrated that the self-attention in Transformers acts as a low-pass filter attenuating high-frequency information, which is a major cause of over-smoothing~\cite{wang2022anti,shi2022revisiting,choi2024GFSA}. AGF mitigates this issue by effectively leveraging various scale frequency information through directly filtering signals in the singular value domain. Fig.~\ref{fig:oversmoothing_filter_response} (a) illustrates the frequency information in both the vanilla DeiT (i.e., $\mathbf{H}=\mathbf{\bar{A}}$) and DeiT + AGF (i.e., $\mathbf{H}=\mathbf{U(X)\Sigma(X)\mathbf{V}(X)^\intercal}$). Unlike the vanilla model, AGF better captures high-frequency information. Additionally, Fig.~\ref{fig:oversmoothing_filter_response} (b) shows the cosine similarity among hidden vectors at each layer. While the cosine similarity in DeiT increases to nearly 0.9 as layers deepen, it is moderated to nearly 0.5 in DeiT + AGF. Thus, AGF prevents over-smoothing in deep Transformers by effectively leveraging diverse frequency information.

\section{Conclusions}
We presented AGF, which interprets the self-attention as learning graph filters in the singular value domain from the perspective of directed graph signal processing. Since the self-attention matrix can be interpreted as a directed graph, we designed a more expressive self-attention using signals directly in the singular value domain. By learning the coefficients for various polynomial bases, AGF uses diverse frequencies. Our experiments showed that AGF outperforms baselines across various tasks, and the training time and GPU usage of AGF are comparable to baseline models with linear complexity. As a side contribution, AGF mitigates the over-smoothing problem in deep Transformers. 

Since our comparison scope is focused on linear Transformers, a limitation is the exploration and comparison regarding the recent state-space models~\cite{gu2021ssm,gu2023mamba}. Exploring the potential of our method in enhancing state-space models is an intriguing avenue for future work.
% Our model suggests a new direction for improving the expressiveness of Transformers while reducing complexity, and we hope to see further research and its applications.

\clearpage

\section*{Acknowledgments}
This work was partly supported by 
Institute for Information \& Communications Technology Planning \& Evaluation (IITP) grants funded by the Korea government (MSIT) (No. RS-2022-II220113, Developing a Sustainable Collaborative Multi-modal Lifelong Learning Framework, and 80\% No. RS-2024-00457882, AI Research Hub Project, 10\%), 
and Samsung Electronics Co., Ltd. (No. G01240136, KAIST Semiconductor Research Fund (2nd), 10\%)

\bibliographystyle{named}
\bibliography{reference}

\clearpage
%%%%%%%%%%%%%%%%%%%%%%%%%%%%%%%%%%%%%%%%%%%%%%%%%%%%%%%%%%%%%%%%%%%%%%%%%%%%%%%

\appendix
\section{Reproducibility Statement}\label{app:reproducibility}
In an effort to ensure reproducibility, we report the description of dataset in Appendix~\ref{app:dataset_description}  and the best hyperparameters of our experiments in Appendix~\ref{app:best_hyperparameter}.
% our AGF code to reproduce the experiment can be found at https://bit.ly/3X5gJ8M.

\section{Limitation}
The computational benefits of the linear complexity might diminish with extremely large datasets, where the overhead of managing large-scale data can still present challenges. The effectiveness of the approach may vary based on the nature of the input data. Data that do not fit well with the assumptions underlying our graph-based approach might result in sub-optimal performance.

\section{Broader Impacts}\label{app:broader_impacts}
Our research focuses on improving the efficiency of Transformers by introducing a mechanism to interpret self-attention as a graph filter and reduce it to linear time complexity. Efficient Transformers can have the positive effect of increasing model accessibility for device deployment and training for research purposes. They can also have environmental benefits by reducing carbon footprint. Therefore, our research has no special ethical or social negative implications compared to other key components of deep learning.

\section{Proof of Theorem~\ref{thm:sa_is_lpf}}\label{app:proof_sa_is_lpf}

\begin{customthm}[Self-attention is a low-pass filter]
Let $\mathbf{M}=\textrm{softmax}(\mathbf{Z})$ for any matrix $\mathbf{Z}\in\mathbb{R}^{n\times n}$. Then $\mathbf{M}$ inherently acts as a low pass filter. For all $\mathbf{x}\in \mathbb{R}^N$, in other words, $\text{lim}_{t\rightarrow \infty}\Vert \text{HFC}[\mathbf{M}^t(\mathbf{x})]\Vert_2/\Vert \text{LFC}[\mathbf{M}^t(\mathbf{x})]\Vert_2=0$
\end{customthm}

\begin{proof}
Given the matrix $\mathbf{M}$ normalized by softmax, $\mathbf{M}$ is a positive matrix and the row-sum of each element in $\mathbf{M}$ is equal to 1. Let the Jordan Canonical Form of $\mathbf{M}$ as $\mathbf{J}$ with the similarity transformation represented by the matrix $P$ as follows:
\begin{align}
    \mathbf{M} = \mathbf{PJP}^{-1},
\end{align}
where $\mathbf{J}$ is block diagonal with each block corresponding to an eigenvalue and its associated Jordan chains. According to the Perron-Frobenius theorem, the largest eigenvalue is real, non-negative and dominant. We denote this eigenvalue as $\lambda_1$.

Consider now the repeated application of $\mathbf{M}$:
\begin{align}
f^t(\mathbf{x}) = \mathbf{M}^t \mathbf{x} = (\mathbf{P} \mathbf{J} \mathbf{P}^{-1})^t \mathbf{x}.
\end{align}
By expanding this expression using the binomial theorem and considering the structure of Jordan blocks, we observe that for large $t$, the dominant term will be $\lambda_1$.  In comparison to the term with $\lambda_1$, other terms that involve smaller eigenvalues or higher powers of $t$ in Jordan blocks will become negligible over time.

In the frequency domain, expressing the transformation shows that high-frequency components attenuate faster than the primary low-frequency component. This occurs because the term $\lambda_1$ becomes overwhelmingly dominant as $t$ increases, causing other components to diminish.

Thus, within the context of our filter definitions, it becomes evident that:
\begin{align}
\lim_{t \rightarrow \infty} \frac{||\text{HFC}[f^t(\mathbf{x}) - \lambda_1^t \mathbf{v}_1]||_2}{||\text{LFC}[\lambda_1^t \mathbf{v}_1]||_2} = 0
\end{align}
Here, $\mathbf{v}_1$ is the generalized eigenvector corresponding to $\lambda_1$.

This behavior indicates a characteristic of a low-pass filter, reaffirming the low-pass nature of $\mathbf{M}$. Importantly, this is independent of the specific configurations of the input matrices $\mathbf{Z}$.
\end{proof}

\section{Proof of Theorem~\ref{thm:coeff}}\label{app:proof_coeff}
\begin{customthm2}
Assume the graph $G$ is connected. Let $\lambda_1\geq\lambda_2 \geq ... \geq  \lambda_n$ be the singular values of $\mathbf{\bar{A}}$. If $\theta_k \geq 0$ for $\forall k \in \{0,1,...,K\}$, $\sum_{k=0}^K \theta_k=1$ and $\exists k'> 0$ such that $\theta_{k'}> 0$, then $\left| g_\theta(\lambda_i)/g_\theta(\lambda_1)\right|<1$ for $\forall i \geq 2$. Also, if $\theta_k=(-\alpha)^k$ , $\alpha \in (0, 1)$ and $K \rightarrow \infty$, then $\left| \text{lim}_{K \rightarrow \infty} g_\theta(\lambda_i)/\text{lim}_{K \rightarrow \infty} g_\theta(\lambda_1) \right| > 1$ for $\forall i \geq 2$.
\end{customthm2}

\begin{proof}
For unfiltered case, the singular value ratio is $\left|g_\theta(\lambda_i)/g_\theta(\lambda_1)\right|=1$. Note that $\left|g_\theta(\lambda_i)/g_\theta(\lambda_1)\right|<1$ for $\forall i \geq 2$ implies that after applying the graph filter, the lowest frequency component $\lambda_1$  further dominates, which indicates the graph filter $g_\theta$ act as a low-pass filter. In contrast $\left|{\text{lim}_{K\rightarrow\infty} g_\theta(\lambda_i)}/{\text{lim}_{K\rightarrow\infty} g_\theta(\lambda_1)}\right|>1$ for $\forall i \geq 2$ indicates the lowest frequency component no longer dominates. This correspond to the high-pass filter case.

For low-pass filter result, from basic spectral analysis~\cite{von2007spectraltutorial}, we know that $\lambda_1=1$ and $\left|\lambda_i\right|<1$ for $\forall i \geq 2$.
By the assumption we know that

\begin{align}
    g_\theta(\lambda_1)=\sum_{k=0}^K \theta_k = 1
\end{align}

Then proving the low-pass filter results is equivalent to show $g_\theta(\lambda_i)<1$ for $\forall i \geq 2$. Since $\theta$ contains non-negative values and for $\forall k \geq 2$, $|\lambda|^k < 1$ because $|\lambda| < 1$, we have:

\begin{align}
g_\theta(\lambda_i)\leq \sum_{k=0}^{K}\theta_k|\lambda^k|=\sum_{k=0}^{K}\theta_k|\lambda|^k\overset{(a)}{\leq} \sum_{k=0}^{K}\theta_k = 1.
\end{align}

Since by assumption $\sum_{k=0}^K \theta_k=1$ and $\exists k' >0$ such that $\theta_{k'}>0$, (a) is a strict inequality $<$. Therefore, the ratio $\left|{g_\theta(\lambda_i)}/{g_\theta(\lambda_1)}\right| < 1$ for $\forall i \geq 2$, we prove the low-pass filter result.

For high-pass filter result, it is clarified that

\begin{align}
\text{lim}_{K\rightarrow\infty} g_\theta(\lambda)& =\text{lim}_{K\rightarrow\infty} \sum_{k=0}^K \theta_k\lambda^k\\&=\text{lim}_{K\rightarrow\infty} \sum_{k=0}^K (-\alpha\lambda)^k \\&=\frac{1}{1+\alpha\lambda}
\end{align}

Thus we have:
    
\begin{align}
\left|\frac{\text{lim}_{K\rightarrow\infty} g_\theta(\lambda_i)}{\text{lim}_{K\rightarrow\infty} g_\theta(\lambda_1)}\right|=\left| \frac{1+\alpha}{1+\alpha\lambda_i}\right|>1 
\end{align} for $\forall i \geq 2$.

Therefore, when the graph filter contains $\theta_k=(-\alpha)^k$ emphasizes the high-frequency components and acts as a hgih-pass filter.

This proof demonstrates that the characteristic of the grah filter, whether as a low-pass of high-pass filter, directly depends on the sign and values of the coefficients.
\end{proof}

\section{Jacobi Basis}
\label{app:jacobi}
Jacobi basis $\mathbf{B}_k^{a,b}(x)$ is recursively defined as follows:

\begin{align}    
\mathbf{B}_0^{a,b}(x)=1, \ 
\mathbf{B}_1^{a,b}(x)=\frac{a-b}{2}+\frac{a+b+2}{2}x.
\end{align}

For $k\geq 2$, it is defined as

\begin{align}
    \mathbf{B}_k^{a,b}(x)=(w_kx+w_k')\mathbf{B}_{k-1}^{a,b}(x)-w_k''\mathbf{B}_{k-2}^{a,b}(x),
\end{align} where

\begin{align}
    w_k &=\frac{(2k+a+b)(2k+a+b+1)}{2k(k+a+b)}, \\
    w_k' &=\frac{(2k+a+b-1)(a^2-b^2)}{2k(k+a+b)(2k+a+b-2)}, \\
    w_k'' &=\frac{(k+a-1)(k+b-1)(2k+a+b)}{k(k+a+b)(2k+a+b-2)},
\end{align}

The Jacobi bases $\mathbf{B}_k^{a,b}$ for $k=0,\cdots,j$ are orthogonal with respect to the weight function $(1-x)^a(1+x)^b$ in the interval $[-1, 1]$ with $a,b>-1$. Therefore, we use the Jacobi basis to stabilize the training of the coefficients.

% \section{Graph Signal Processing on Directed Graph}
\section{Dataset Description}\label{app:dataset_description}

\subsection{Long Range Arena Benchmark}\label{app:lra}
\begin{table}[h]
    % \footnotesize
    \centering
    \begin{tabular}{l cccc}
    \toprule
        \bfseries  Datasets &  \#Train & \#Test & Lengths & \#Classes  \\ \midrule
        Listops      & 96,000 & 2,000 & 2K & 10 \\ 
        Text      & 25,000 & 25,000 & 4K & 2 \\ 
        Retrieval      & 147,086 & 17,437 & 4K & 2 \\ 
        Image      & 45,000 & 10,000 &  1K & 10 \\ 
        Pathfinder      & 160,000 & 20,000 &   1K & 2 \\ 
    \bottomrule     
    \end{tabular}
    \caption{Statistics of Long Range Arena benchmark datasets}
    % }
    \label{tab:lra_dataset}
\end{table}

We describe the statistics of Long Range Arena benchmark, including equation calculation (ListOps)~\cite{nangia2018listops}, review classification (Text)~\cite{maas2011text}, document retrieval (Retrieval)~\cite{radev2013retrieval}, image classification (Image)~\cite{krizhevsky2009image}, and image spatial dependencies (Pathfinder)~\cite{linsley2018pathfinder}. Long Range Arena (LRA) is a benchmark to systematically evaluate efficient Transformer models. The statistics of datasets are summarized in Table~\ref{tab:lra_dataset}. Also, the descriptions of datasets are as follows:

\begin{itemize}
    \item ListOps: The consists of summarization operations on a list of single-digit integers written in prefix notation. The entire sequence has a corresponding solution that is also a single-digit integer. Target task is a 10-way balanced classification problem.
    \item Text: The byte/character-level setup in order to simulate a longer input sequence. This task uses the IMDb reviews dataset, which is a commonly used dataset to benchmark document classification, with  a fixed max length of 4K. This is a binary classification task with accuracy.
    \item Retrieval: To evaluate the ability to encode and store compressed representations that are useful for matching and retrieval, the task is to learn the similarity score between two vectors. This is a binary classification task with accuracy as the metric.
    \item Image: This task is an image classification task, where the inputs are sequences of pixels which is flatten to a sequence of length as $n^2$. CIFAR-10 is used to image classfication.
    \item Pathfinder: The task evaluates a model to make a binary decision whether two points represented as circles are connected by a path consisting of dashes.
\end{itemize}

\subsection{UEA Time Series Classification}\label{app:uea}

\begin{table}[h]
    % \footnotesize
    \small
    \centering
    \setlength{\tabcolsep}{1.8pt}
    \begin{tabular}{l ccccc}
    \toprule
        \bfseries  Datasets &  \#Train & \#Test & \#Features & Lengths & \#Classes  \\ \midrule
        EthanolConcentration     & 261 & 263 & 3 & 1,751 & 4 \\ 
        FaceDetection            & 5,890 & 3,524 & 144 & 62 & 2 \\ 
        HandWriting              & 150 & 850 & 3 & 152 & 26 \\ 
        Heartbeat                & 204 & 205 & 61 & 405 & 2 \\ 
        JapaneseVowels           & 270 & 370 & 12 & 29 & 9 \\ 
        PEMS-SF                  & 267 & 173 & 963 & 144 & 7 \\ 
        SelfRegulationSCP1       & 268 & 293 & 6 & 896 & 2 \\ 
        SelfRegulationSCP2       & 200 & 180 & 7 & 1,152 & 2 \\ 
        SpokenArabicDigits       & 6,599 & 2,199 & 13 & 93 & 10 \\ 
        UWaveGestureLibrary      & 120 & 320 & 3 & 315 & 8 \\ 
    \bottomrule     
    \end{tabular}
    \caption{Statistics of UEA time series classification benchmark datasets}
    % }
    \label{tab:uea_dataset}
\end{table}

We describe the statistics of the UEA time series classification datasets in Table~\ref{tab:uea_dataset}.
\begin{itemize}
    \item EthanolConcentration: A dataset of raw spectra of water-and-ethanol solu- tions in 44 distinct, real whisky bottles.
    \item FaceDetection: A dataset consisting of MEG recordings and class labels (Face/Scramble)
    \item HandWriting: A motion dataset taken from a smartwatch while the subject writes 26 letters of the alphabet generated by UCR.
    \item Heartbeat: Recordings of heart sounds collected from both healthy subjects and pathological patients in clinical or non-clinical settings.
    \item JapaneseVowels: A dataset where nine Japanese-male speakers were recorded saying the vowels ‘a’ and ‘e’.
    \item PEMS-SF: A UCI dataset from the California Department of Transportation19 with 15 months worth of daily data from the California Department of Transportation PEMS website. 
    \item SelfRegulationSCP1: A dataset is Ia in BCI II 12: Self-regulation of slow cortical potentials. 
    \item SelfRegulationSCP2: A dataset is Ib in BCI II 12: Self-regulation of slow cortical potentials. 
    \item SpokenArabicDigits: A dataset taken from 44 males and 44 females Arabic native speakers between the ages 18 and 40 to represent ten spoken Arabic digits.
    \item UWaveGestureLibrary: A set of eight simple gestures generated from accelerometers.
\end{itemize}

\section{Experimental Setting}\label{app:exp_setting}
\subsection{UEA Time Series Classification}\label{app:exp_setting_uea}
We compare our model with vanilla Transformer and linear complexity Transformers, including Reformer~\cite{kitaev2020reformer}, Linformer~\cite{wang2020linformer}, Nystr{\"o}mformer~\cite{xiong2021nystromformer}, Longformer~\cite{beltagy2020longformer}, YOSO-E~\cite{zeng2021yosoe}, and PirmalFormer~\cite{chen2023primal}. We grid search the order of polynomial $K$ in \{2,3,4,5\}, $a$ in \{1.0, 1.5, 2.0\}, $b$ in \{-2.0, -1.5, $\cdots$, 1.5, 2.0\}, and $\gamma$ in \{\num{1e-01}, \num{1e-02}, \num{1e-03}, \num{1e-04}\}.

\subsection{Long Range Arena Benchmark}\label{app:exp_setting_lra}
We compare our model with vanilla Transformer and linear complexity Transformers, including LinearTransformer~\cite{katharopoulos2020lineartransformer}, Reformer~\cite{kitaev2020reformer}, Performer~\cite{choromanski2020performer}, Longformer~\cite{beltagy2020longformer}, YOSO-E~\cite{zeng2021yosoe}, Cosformer~\cite{qin2022cosformer}, SOFT~\cite{lu2021soft}, Flowformer~\cite{wu2022flowformer}, and PirmalFormer~\cite{chen2023primal}. We grid search the order of polynomial $K$ in \{2,3,4,5\}, $\gamma$ in \{\num{1e-01}, \num{1e-02}, \num{1e-03}, \num{1e-04}\}, and learning rate $\eta$ in \{\num{1e-3},\num{5e-4},\num{1e-4}, \num{5e-5}\}.

\subsection{Vision Transformer}\label{app:exp_setting_vit}
Many recent works do not apply the proposed method to all layers, but adjust the number of layers to apply the proposed method~\cite{patro2023spectformer,chen2023primal}.
Therefore, we set the number of applied layers $L$ as a
hyperparameter. We conduct experiments with grid search on the order of polynomial $K$ in \{3,4\}, $\gamma$ in \{\num{1e-03}, \num{1e-04}\}, and the number of layers $L$ in \{1, 3, 6\}.

\begin{table}[h!]
    \centering
    \small
    \begin{tabular}{l cccc}
    \toprule
          Datasets &  $K$ & $\gamma$  & $a$ & $b$ \\ \midrule
        EthanolConcentration     & 6 & 0.1 & 1.5 & -1.5  \\ 
        FaceDetection            & 9 & 0.1 & 2.0 & 0.5  \\ 
        HandWriting              & 3 & 0.01 & 1.0 & 0.2   \\ 
        Heartbeat                & 6 & 0.01 & -0.5 & -0.5   \\ 
        JapaneseVowels           & 4 & 0.01 & 0.0 & 0.0   \\ 
        PEMS-SF                  & 5 & 0.1 & 1.5 & 2.0   \\ 
        SelfRegulationSCP1       & 9 & 0.1 & 1.5 & 2.0   \\ 
        SelfRegulationSCP2       & 7 & 0.1 & 2.0 & 1.5   \\ 
        SpokenArabicDigits       & 4 & 0.1 & 1.0 & 0.0   \\ 
        UWaveGestureLibrary      & 10 & 0.1 & 2.0 & 1.5   \\ 
    \bottomrule     
    \end{tabular}
    \caption{Best hyperparameters for AGF in UEA time series classification}
    \label{tab:uea_hyperparameter}
\end{table}
			
\begin{table}[h!]
\centering
        \small
        \setlength{\tabcolsep}{3pt}
        \begin{tabular}{l ccccc}
        \toprule
             Datasets  & $K$ & $\gamma$ & $a$ & $b$ & $\eta$ \\ \midrule
     Text &  5 & 0.001 & 1.0 & 1.0 & \num{5e-05}   \\
     Image &   7 & 0.01 & 2.0 & -1.0 & \num{1e-3}  \\
     Retrieval &   5 & 0.1 & 1.0 & 1.0 & \num{5e-4}  \\
     ListOps &   5 & 0.1 & 1.5 & 1.0 & \num{1e-3}  \\
     Pathfinder &   3 & 0.0001 & 1.5 & -1.0 & \num{1e-4}  \\
        \bottomrule     
        \end{tabular}
        \caption{Best hyperparameters for AGF in Long Range Arena benchmark}
        \label{tab:lra_hyperparameter}
        
\end{table}
\begin{table}[h!]
    \setlength{\tabcolsep}{2.5pt}
    \centering
    \small
    \begin{tabular}{l ccccccc}
    \toprule
        Datasets  & $K$ & $\gamma$ & $L$ & a & b\\ \midrule
        ImageNet-100     & 4 & 0 & 1 & -0.5 & 0.0 \\ 
        ImageNet-1k      & 3 & 1e-4 & 6 & -0.5 & 0.0 \\ 
    \bottomrule     
    \end{tabular}
    \caption{Best hyperparameters for AGF in image classification}
    \label{tab:vit_hyperparameter}
\end{table}

\section{Best Hyperparameters}\label{app:best_hyperparameter}
Tables~\ref{tab:uea_hyperparameter}-\ref{tab:vit_hyperparameter} show the best hyperparameters used in our experiments.

\section{Significance Test}
We perform a significance test on the UEA time series classification results, as shown in Table~\ref{tab:sig_uea}. Specifically, we use the Wilcoxon signed-rank test, a non-parametric statistical method for comparing differences between two related samples. A $p$-value below 0.05 indicates that the result is statistically significant, suggesting that the observed differences are unlikely to be due to chance.

\begin{table}[!h]
    \centering
    \small
    \begin{tabular}{l ccccc}
    \toprule
    \multirow{4}{*}{AGF v.s.}   & Trans. & Linear. & Re. & Long. & Per.  \\ 
               & 0.002  & 0.010 & 0.002    & 0.049      & 0.002       \\  \cmidrule{2-6}
               & YOSO-E & Cos. & SOFT & Flow. & Primal. \\
               & 0.002  & 0.002     & 0.01 & 0.004      & 0.008 \\ \bottomrule  
\end{tabular}
\caption{$p$-value using Wilcoxon signed-rank}
\label{tab:sig_uea}
\end{table}

\section{Sensitivity Studies}\label{app:full_sens}
In this section, we report the sensitivity studies on UEA time series classification and LRA benchmark for all datasets. Tables~\ref{tab:uea_k_sens_full} and~\ref{tab:lra_k_sens_full} display the comprehensive results of the impact of the polynomial order $K$. 

\subsection{Effect of the Order of Polynomial $K$}\label{app:sens_K}

\begin{table}[!h]
    % \footnotesize
    \setlength{\tabcolsep}{1.pt}
    \centering
    \begin{tabular}{l ccccccccc}
    \toprule
    $K$     & 3 & 4 & 5 & 6 & 7 & 8 & 9 & 10 \\ \midrule
EC      & 32.3  & 31.6  & 33.5  & \textbf{36.1}  & 31.9  & 29.7  & 30.8  & 32.3  \\
FD      & 68.2  & 68.8  & 68.1  & 68.3  & 68.2  & 68.3  & \textbf{69.9}  & 67.5  \\
HW      & \textbf{33.5}  & 29.9  & 33.4  & 33.2  & 33.2  & 32.6  & 32.5  & 29.6  \\
HB      & 77.6  & 77.6  & 77.6  & \textbf{79.0}  & 77.6  & 77.6  & 77.6  & 77.6  \\
JV      & 98.9  & \textbf{99.5}  & 98.9  & 98.9  & 99.2  & 98.9  & 99.2  & 98.9  \\
PEMS-SF & 86.7  & 89.6  & \textbf{91.3}  & 83.8  & 86.1  & 85.0  & 83.8  & 87.3  \\
SRSCP1  & 91.1  & 92.2  & 89.4  & 91.1  & 90.8  & 89.4  & \textbf{93.5}  & 91.1  \\
SRSCP2  & 55.6  & 55.6  & 56.1  & 56.7  & \textbf{58.9}  & 56.1  & 56.1  & 56.1  \\
SAD     & 100.0 & 100.0 & 100.0 & 100.0 & 100.0 & 100.0 & 100.0 & 100.0 \\
UWGL    & 84.1  & 84.1  & 86.9  & 84.4  & 84.1  & 85.0  & 86.2  & \textbf{89.4} \\
    \bottomrule     
    \end{tabular}
    \caption{Effect of $K$ on UEA Time Series classification} 
    \label{tab:uea_k_sens_full}
    % }
\end{table}

\begin{table}[!h]
    % \footnotesize
    \centering
    \begin{tabular}{l ccccccc}
    \toprule
    $K$ & ListOps & Text & Retrieval & Image & Pathfinder  \\ \midrule
        3.0 & 37.0 & 64.5 &  80.6 & 42.3  & 74.0 \\
       4.0 & 37.5 & 64.3 &  79.7 & 40.8 & 74.0 \\ 
     5.0 & 38.0 & 64.7 & 81.4 & 41.4 & 72.4 \\
     6.0 & 37.0 & 64.5 & 80.6 & 42.3 & 71.4 \\
      7.0 & 37.3 & 63.9 & 80.0 & 42.4 & 72.6  \\
    \bottomrule     
    \end{tabular}
    \caption{Effect of $K$ on Long Range Arena benchmark} 
    \label{tab:lra_k_sens_full}
    % }
\end{table}

We report the sensitivity results for the order of polynomial $K$ on UEA time series classification and LRA benchmark for all datasets. The results are summarized in Tables~\ref{tab:uea_k_sens_full} and ~\ref{tab:lra_k_sens_full}.

\subsection{Effects on the $a$ and $b$ of Jacobi Polynomial}\label{app:sens_a_b}

We conduct the senstivity studies on the effect of the parameters $a$ and $b$ in Jacobi polynomial. Tables~\ref{tab:lra_a_sens} and~\ref{tab:lra_b_sens} summarize the results on UEA time series classification and LRA benchmark for all datasets.

\begin{table}[!h]
    \centering
    \setlength{\tabcolsep}{1.5pt}
    \begin{tabular}{l ccccc}
    \toprule
    $a$ & ListOps & Text & Retrieval & Image & Pathfinder \\ \midrule
      1.0 & 36.8 & 64.7 & 81.4 & 41.8 & 72.5 \\
      1.5 & 38.0 & 64.3 & 80.7 & 41.6 & 74.0 \\
      2.0 & 36.9 & 64.2  & 80.4 & 42.4 & 73.4 \\
    \bottomrule     
    \end{tabular}
    \caption{Effect of $a$ on LRA benchmark} 
    \label{tab:lra_a_sens}
\end{table}

\begin{table}[!h]
    \small
        \centering
        \setlength{\tabcolsep}{1.5pt}
        \begin{tabular}{l ccccc}
        \toprule
        $b$ & ListOps & Text & Retrieval & Image & Pathfinder\\ \midrule
      -2.0 & 37.1 & 64.1 &  81.3 & 40.9  & 73.4 \\
  -1.0 & 37.6 & 64.3 & 80.8 & 42.4 & 74.0 \\
    1.0 & 38.0 & 64.7 & 81.4 & 41.6 & 70.6  \\
       2.0 & 37.0 & 64.5 &  81.1 & 42.3 &  72.7 \\
        \bottomrule     
        \end{tabular}
        \caption{Effect of $b$ on LRA benchmark} 
        \label{tab:lra_b_sens}
\end{table}

\section{Ablation Studies}\label{app:ablation}

\subsection{Effect of Activation Function $\rho$}
\begin{table}[h]
    \centering
    \small
    \setlength{\tabcolsep}{2pt}
    \begin{tabular}{l ccccc }
    \toprule
      & ListOps &  Text & Retrieval & Image & Pathfinder \\ \midrule
      %  2K 4K 4K 1K 1K 
        None     & 17.8 & 63.4 & 75.3 & 41.1 & 70.6 \\
        Sigmoid  & 17.7 & 63.1 & 66.3 & 40.7 & 72.3 \\
        Tanh     & 17.8 & 63.6 & 77.2 & 42.2 & 69.0  \\
        \textbf{Softmax} & 38.0 & 64.7 & 81.4  & 42.4 & 74.0 \\
        \bottomrule     
    \end{tabular}
    \caption{Ablation study on $\rho$ in Eqs.~\eqref{eq:singular} and~\eqref{eq:singular2}} 
    \label{tab:abla_act}
\end{table}
We conduct an ablation study on the activation function $\rho$ applied to the learnable singular vectors $U(\mathbf{X})$ and $V(\mathbf{X})$. In Table~\ref{tab:abla_act}, `None' indicates that no activation function is applied to the generated singular vectors, `Sigmoid' and `Tanh' correspond to the use of the sigmoid and hyperbolic tangent functions, respectively, and `Softmax' represents the softmax function used in our proposed model. Consistently, the softmax function demonstrates optimal performance. Since the generated singular vectors serve as the basis for signals in the graph filter, the choice of an appropriate activation function significantly impacts the learning process of the graph filter.

\subsection{Effect of Polynomial Type}

\begin{table*}[t]
\centering
\resizebox{0.98\textwidth}{!}{%
\label{tab:abl_poly_type_full}
\begin{tabular}{l  cccccccccc  c}\toprule
                  & \multicolumn{1}{l}{EC} & \multicolumn{1}{l}{FD} & \multicolumn{1}{l}{HW} & \multicolumn{1}{l}{HB} & \multicolumn{1}{l}{JV} & \multicolumn{1}{l}{PEMS-SF} & \multicolumn{1}{l}{SRSCP1} & \multicolumn{1}{l}{SRSCP2} & \multicolumn{1}{l}{SAD} & \multicolumn{1}{l}{UWGL} & \multicolumn{1}{c}{Average} \\ \midrule
 
AGF(Monomial)               & 29.7 & 67.0 & 32.4 & 77.6 & 98.6  & 86.7 & 90.8 & 55.6 & 100.0 & 83.8 & 72.2 \\
AGF(Jacobi)               & \textbf{36.1} & \textbf{69.9} & \textbf{33.5} & \textbf{79.0}& \textbf{99.5} & \textbf{91.3} & \textbf{93.2} & \textbf{58.9} & \textbf{100.0} & \textbf{89.4} & \textbf{75.1} \\ \bottomrule
\end{tabular}
}
\caption{Ablation study on polynomial type}
\end{table*}

We perform ablation studies on the type of polynomial $T$ and results are shown in Table~\ref{tab:abla_act}. `Monomial' indicates that we adopt the monomial basis, and `Jacobi' means the Jacobi basis as used in AGF. The Jacobi basis consistently outperforms the Monomial basis, which indicates that the orthogonal basis is stable in terms of convergence and leads to better performance.

\subsection{Effect of the Graph Filter}
\begin{table*}[!h]
\centering
\begin{tabular}{l  cccccccccc  c}\toprule
                  & \multicolumn{1}{l}{EC} & \multicolumn{1}{l}{FD} & \multicolumn{1}{l}{HW} & \multicolumn{1}{l}{HB} & \multicolumn{1}{l}{JV} & \multicolumn{1}{l}{PEMS-SF} & \multicolumn{1}{l}{SRSCP1} & \multicolumn{1}{l}{SRSCP2} & \multicolumn{1}{l}{SAD} & \multicolumn{1}{l}{UWGL} & \multicolumn{1}{c}{Average} \\ \midrule

$\mathbf{H}_{UV^\intercal}$  & 29.7 & 66.6 & 28.2 & 76.6 & 98.4 & 87.3 & 90.8 & 56.1 & 99.9 & 83.8 & 71.7 \\
$\mathbf{H}_{SVD}$ & 33.1 & 67.1 & 27.1 & 75.1 & 98.4  &   88.4  &  89.8  &  56.1 & 100.0 & 85.9 &  72.1\\
AGF            & \textbf{36.1} & \textbf{69.9} & \textbf{33.5} & \textbf{79.0}& \textbf{99.5} & \textbf{91.3} & \textbf{93.2} & \textbf{58.9} & \textbf{100.0} & \textbf{89.4} & \textbf{75.1} \\ \bottomrule
\end{tabular}
\caption{Ablation study on the graph filter}
\label{tab:abl_sigma_full}
\end{table*}

We conduct an ablation study on the signal filtering in main text, especially in subsection~\ref{subsec:abla}. In this subsection, we show the full results in Table~\ref{tab:abl_sigma_full}. our AGF with appropriate graph signal filtering enhances the self-attention in Transformers, resulting in a significant improvement in average performance.

\section{Empirical Runtime}\label{app:runtime_uea}

\begin{table*}[h]
    \small
    \centering
    \setlength{\tabcolsep}{0.8pt}
    \begin{tabular}{l cccccccccc   }\toprule
    & EC & FD & HW & HB & JV & PEMS-SF & SRSCP1 & SRSCP2 & SAD & UWGL \\ \midrule
    Trans.  & 3.7/13.88 & 7.3/0.21 & 0.2/0.44 & 0.5/1.23 & 0.3/0.14 & 0.4/0.41 & 1.5/4.05 & 1.5/6.36 & 3.0/0.22 & 0.2/0.91\\
    Flow. & 2.1/3.75 & 8.4/0.24 & 0.4/0.47 & 0.5/0.95 & 0.6/0.16 & 0.4/0.43 & 1.2/1.96 & 1.1/2.49 & 4.8/0.25 & 0.3/0.78 \\ 
    Primal. & 2.1/3.37  & 10.8/0.22  & 0.3/0.44 & 0.5/0.89 & 0.6/0.15 & 0.5/0.43 & 1.3/1.78 & 1.1/2.16 & 3.2/0.23 & 0.2/0.71 \\ \midrule
    AGF     & 2.7/4.79 & 10.6/0.32 & 0.3/0.52 & 0.5/0.99  & 0.6/0.18 & 0.5/0.51 & 1.7/2.27 & 1.5/3.31 & 4.5/0.28 & 0.3/1.14 \\ \bottomrule
    \end{tabular}
    \caption{Comparison of running time (s/epoch) and the peak training memory usage (GB) on UEA time series benchmark}
    \label{tab:runtime_uea}
\end{table*}

We report the empirical runtime on UEA time series classification and the result is reported in Table~\ref{tab:runtime_uea}. It shows comparable efficiency when compared to existing transformers and transformers with linear complexity. The results for LRA benchmark is reported in main text, specifically in Table~\ref{tab:runtime_lra}.

\end{document}